\documentclass[a4paper]{article}
\usepackage{amsmath,amstext,amsgen,amsbsy,amsopn,amsfonts,amssymb}
\usepackage{etex}
\usepackage{graphics}
\usepackage[pdftex]{graphicx}
\usepackage{epsfig}
\usepackage{hyperref}
\usepackage{amsthm}
\usepackage{float}
\usepackage{bm}
\usepackage{color}
\usepackage{algorithmic}
\usepackage{algorithm}
\usepackage{quoting}
\usepackage{flushend}
\usepackage{balance}
\usepackage{multicol}
\usepackage[resetlabels]{multibib}
\newtheorem{theorem}{Theorem}[section]

\usepackage{booktabs}
\usepackage{esvect}
\usepackage{multirow}
\usepackage{array}
\usepackage{arydshln}
\usepackage{xcolor}
\usepackage{hyperref}

\providecommand{\keywords}[1]{\textbf{\textit{Index terms---}} #1}

\begin{document}
\title{\textbf{Similarity Kernel and Clustering via Random Projection Forests}}

\author{
Donghui Yan$^{\dag}$, Songxiang Gu$^{\ddag}$, Ying Xu$^{\P}$, Zhiwei Qin$^{\$}$
\vspace{0.1in}\\
$^\dag$Mathematics, University of Massachusetts Dartmouth, MA\vspace{0.05in}\\
$^\ddag$Linkedin Inc, Sunnyvale, CA\vspace{0.05in}\\
$^\P$Indigo Agriculture Inc, Boston, MA\vspace{0.05in}\\
$^\$$DiDi Research America, Mountain View, CA\vspace{0.05in}
%%\\
}

\date{\today}
\maketitle

\begin{abstract}
\noindent
Similarity plays a fundamental role in many areas, including data mining, machine learning, statistics and various 
applied domains. Inspired by the success of ensemble methods and the flexibility of trees, we propose to learn a similarity 
kernel called {\it rpf-kernel} through random projection forests ({\it rpForests}). Our theoretical analysis reveals a highly desirable 
property of rpf-kernel: {\it far-away (dissimilar) points have a low similarity value while nearby (similar) points would have 
a high similarity}, and the similarities have a native interpretation as the probability of points remaining in the same leaf nodes 
during the growth of {\it rpForests}. The learned rpf-kernel leads to an effective clustering algorithm---{\it rpfCluster}. On 
a wide variety of real and benchmark datasets, {\it rpfCluster} compares favorably to K-means clustering, spectral clustering 
and a state-of-the-art clustering ensemble algorithm---Cluster Forests. Our approach is simple to implement and readily adapt 
to the geometry of the underlying data. Given its desirable theoretical property and competitive empirical performance when 
applied to clustering, we expect rpf-kernel to be applicable to many problems of an unsupervised nature or as a regularizer 
in some supervised or weakly supervised settings.    
\end{abstract}


\section{Introduction}
\label{section:introduction}
% no \IEEEPARstart
Similarity measures how similar or closely related different objects are. It is a fundamental quantity that is relevant whenever a distance metric 
or a measure of relatedness is used. Our interest is when the similarity is explicitly used. A wide range of applications use the notion of 
similarity. In data mining, the answer to a database (or web) query is often formulated as similarity search \cite{DongMosesLi2011, 
FriedmanBentleyFinkel1977} where instances similar to the one in the query are returned. More general is the k-nearest neighbor 
search, which has applications in robotic route planning \cite{Kleinbort2015EfficientHM}, face-based surveillance systems 
\cite{Otto2017ClusteringMOArxiv}, anomaly detection \cite{AngiulliPizzu2002,ChandolaBanerjeeKumar2007,RamaswamyRastogiShim2000} etc. 
In machine learning, the entire class of kernel methods \cite{ScholkopfSmola2001,HofmannScholkopfSmola2008} rely on similarity or the 
similarity kernel. The similarity kernel is also an essential part of popular classifiers such as support vector machines \cite{CortesVapnik1995}, 
the broad class of spectral clustering algorithms \cite{Ncut, NgJordan2002, KannanVempala2004,Luxburg2007, YanHuangJordan2009}, 
and various kernalized methods 
such as kernel PCA \cite{Scholkopf1998}, kernel ICA \cite{BachJordan2003}, kernel CCA \cite{FukumizuBachGretton2007}, kernel k-means 
\cite{DhillonGK2004} etc. Another important class of methods that use similarity is clustering ensemble \cite{StrehlGhosh2002,CF} which use 
the similarity to encode how different data points are related to each other in terms of clustering when viewed from many different clustering instances. 
Additionally, the similarity is used as a regularization term that incorporates some latent structures of the data such as the {\it cluster} \cite{ChapelleWeston2003} 
or {\it manifold} assumption \cite{BelkinNiyogiSindhwani2006} in semi-supervised learning, or as regularizing features that capture the locality of 
the data (i.e., which of the data points look similar) \cite{deepTacoma2019}. 
In statistics, the notion of similarity is also frequently used (a significant 
overlap with those used in machine learning). It has been used as a distance metric in situations where the Euclidean distance is no longer appropriate 
or the distance needs to be adaptive to the data, for example in nonparametric density estimation \cite{BickelBreiman1983,MackRosenblatt1979} and 
intrinsic dimension estimation \cite{BickelYan2008,LevinaBickel2005}. Another use is to measure the relatedness between two random objects,
such as the covariance matrix. 
It is also used as part of the aggregation engine to combine data from multiple sources \cite{SNF2014} or multiple views 
\cite{DingShaoFu2018,XuTaoXu2013}.
\\
\\
The benefit of adopting a similarity kernel is immediate. It encodes important property about the data, and allows a unified treatment of many seemingly 
different methods. Also one may take advantage of the powerful kernel methods. However, in many situations, one has to choose a kernel that is 
suitable for a particular application. It will be highly desirable to be able to automatically choose or learn a kernel that would work for the underlying data. In this paper, we explore a data-driven approach to learn a similarity kernel. The similarity kernel will be learned through a versatile tool 
that was recently developed---random projection forests ({\it rpForests}) \cite{rpForests2019}.
\\
\\
{\it rpForests} is an ensemble of random projection trees (rpTrees) \cite{RPTree} with the possibility of projections selection 
during tree growth \cite{rpForests2019}. rpTrees is a randomized version of the popular kd-tree \cite{Bentley1975,FriedmanBentleyFinkel1977}, 
which, instead of splitting the nodes along coordinate-aligning axes, recursively splits the tree along randomly chosen directions. 
{\it rpForests} combines the power of ensemble methods \cite{Bagging,RF,Adaboost,CF, netflix2012} and the flexibility of trees.
As {\it rpForests} uses randomized trees as its building block, accordingly it has several desired properties of trees. 
Trees are invariant with respect to monotonic transformations of the data. 
Trees-based methods are very efficient with a log-linear (i.e., $O(n\log(n))$) average computational 
complexity for growth and $O(\log(n))$ for search where $n$ is the number of data 
points. As trees are essentially recursive space partitioning methods \cite{Bentley1975, RPTree, YanDavis2018},  
data points falling in the same tree leaf node would be close to each other or ``similar". This property is often leveraged 
for large scale computation \cite{DasguptaSinha2015,LiuMooreGray2004,YanHuangJordan2009, rpForests2019} etc.
Now it forms the basis of our construction of the similarity kernel---data points in the same leaf node are likely to be similar 
and dissimilar otherwise. Additionally, as individual trees are rpTrees thus can adapt to the geometry of the data and readily 
overcomes the curse of dimensionality \cite{RPTree}.
\\
\\
Ideally, for a similarity kernel, similar objects should have a high similarity value while low similarity value for dissimilar objects. 
Similarity as produced by a single tree may suffer from the undesirable boundary effect---similar data points may be separated 
into different leaf nodes during the growth of a tree. 
This will cause problem for a similarity kernel for which the similarity of every pair (or most pairs) of points matters. The 
ensemble nature of {\it rpForests} allows us to effectively overcome the boundary effect---by ensemble, data points separated 
in one tree may meet in another; indeed {\it rpForests} 
reduces the chance of separating nearby points exponentially fast \cite{rpForests2019}. Meanwhile, dissimilar or 
far-away points would unlikely end up in the same leaf node. This is because, roughly, the diameter of tree nodes keeps on 
shrinking during the tree growth, and eventually those dissimilar points would be separated if they are far away enough. Thus 
a similarity kernel produced by {\it rpForests} would possess the expected property.
\\
\\
Our main contributions are as follows. First, we propose a data-driven approach to learn a similarity kernel from the data by {\it rpForests}. 
It combines the power of ensemble and the flexibility of trees; the method is simple to implement, and readily 
adapt to the geometry of the data. As an ensemble method, easily it can run on clustered or multi-core computers. 
Second, we develop a theory on the property of the learned similarity kernel: similar objects would have high similarity value while low 
similarity value for dissimilar objects; the similarity values have a native interpretation as the probability of points staying in the same 
tree leaf node during the growth of {\it rpForests}. With the similarity kernel learned by {\it rpForests}, we develop a highly 
competitive clustering algorithm, {\it rpfCluster}, which compares favorably to spectral clustering and a state-of-the-art ensemble clustering 
method.
\\
\\
The remainder of this paper is organized as follows. In Section \ref{section:method}, we give a detailed description of how to produce 
a similarity kernel by {\it rpForests} and a clustering method based on the resulting similarity kernel. This is followed by a little theory 
on the similarity kernel learned by {\it rpForests} in Section~\ref{section:theory}. 
Related work are discussed in Section \ref{section:related}. In Section~\ref{section:evaluation}, we first provide examples to illustrate the
desired property of the similarity kernel and its relevance in clustering, and then present experimental results on a wide variety of real 
datasets for {\it rpfCluster} and its competitors. Finally, we conclude in Section~\ref{section:conclusion}.
\section{Proposed approach}
\label{section:method}
In this section, we will first describe {\it rpForests}, and then discuss how to generate the similarity kernel with {\it rpForests} and
to cluster with the similarity kernel, followed by a brief introduction to spectral clustering \cite{Ncut, NgJordan2002, KannanVempala2004,Luxburg2007}. 
%%\\ 
%%\\
\subsection{Algorithmic description of {\it rpForests}}
Our description of the {\it rpForests} algorithm is based on \cite{rpForests2019}.
Each tree in {\it rpForests} is an rpTree. The growth of an rpTree follows a recursive procedure. It starts by treating the entire data as the 
root node and then split it into two child nodes according to a randomized splitting rule. On each child node, the same 
splitting procedure applies recursively until some stopping criterion is met, e.g., the node becomes too small (i.e., contains too few data points). 
\\
\\
In rpTree, the split of a node is along a randomly generated direction, say $\stackrel{\rightharpoonup}{r}$. Assume the node to split is $W$. 
There are many ways to split node $W$. One choice that is simple to implement is to select a point, say $c$, uniformly at random over the 
interval formed by the projection of all points in $W$ onto $\stackrel{\rightharpoonup}{r}$, denoted by 
$W_{\stackrel{\rightharpoonup}{r}}=\{P_ {\stackrel{\rightharpoonup}{r}}(x)=r \boldsymbol{\cdot} x:~x\in W\}$. The left child is given by 
$W_L = \{x \in W: P_{\stackrel{\rightharpoonup}{r}}(x) < c\}$, and the right child $W_R$ by the rest of points. 
Another popular choice is to choose the median of $W_{\stackrel{\rightharpoonup}{r}}$ as the split point. We empirically evaluate the 
performance of uniform random split and split by median, and found their difference to be rather small.
Let $V=\{X_1,...,X_n\}$ denote the given data set. Let $\mathcal{W}$ denote the set of nodes to be split (termed as the {\it working set}). 
Let $n_s$ denote a constant such that a node will not be split further if its size is smaller than $n_s$. Denote the {\it rpForests} by $\mathcal{F}$; 
assume there are totally $T$ trees. The algorithm for generating {\it rpForests} is described as Algorithm~\ref{algorithm:rpForest}.
\begin{algorithm}
\caption{\it~~rpForests(V, T)}
\label{algorithm:rpForest}
\begin{algorithmic}[1]
\STATE Initialize $\mathcal{F} \gets \emptyset$; %%
\FOR {$i=1$ to $T$}
\STATE Let $V$ be the root node of tree $t_i$; %%
\STATE Initialize the working set $\mathcal{W} \leftarrow \{V\}$; %%
\WHILE {$\mathcal{W}$ is not empty}
	\STATE Sample $W \in \mathcal{W}$ and update $\mathcal{W} \leftarrow \mathcal{W} \setminus \{W\}$; %%
	\IF{$|W| < n_s$} 
		\STATE Skip to the next round of the loop; %%
	\ENDIF 
    	\STATE Generate a random direction $\stackrel{\rightharpoonup}{r}$;  %%
	\STATE Project $W$ onto $\stackrel{\rightharpoonup}{r}$, $W_{\stackrel{\rightharpoonup}{r}}=\{P_{\stackrel{\rightharpoonup}{r}}(x): x \in W\}$; %%
	\STATE Sample splitting point $c$ uniformly at random from the interval $[\min(W_{\stackrel{\rightharpoonup}{r}}), \max(W_{\stackrel{\rightharpoonup}{r}})]$; %%   
	\STATE Set $W_L \gets \{x: P_{\stackrel{\rightharpoonup}{r}}(x) < c\}$ and $W_R \gets \{x: P_{\stackrel{\rightharpoonup}{r}}(x) \geq c\}$; %%
	\STATE Split node $W$ by $W = W_L \cup  W_R$; %%
	\STATE Update working set by $\mathcal{W} \leftarrow \mathcal{W} \cup \{W_L, W_R\}$; %%
\ENDWHILE%%
\STATE Add tree $t_i$ to the ensemble $\mathcal{F} \gets \mathcal{F} \cup \{t_i\}$;%%
\ENDFOR
\STATE return($\mathcal{F}$); %%
\end{algorithmic}
\end{algorithm} 
%%
%%
%%
%%
%%\\
%%\\
Note that in choosing the splitting direction, we follow the {\it basic implementation} of {\it rpForests} as documented in \cite{rpForests2019} 
which would simply generate a random direction. 
\subsection{Similarity kernel and clustering}
In this subsection, we will describe algorithms for the generation of a similarity kernel with {\it rpForests} and for 
clustering with the resulting similarity kernel. 
\\
\\
Once {\it rpForests} is grown, the generation of the similarity kernel is fairly straightforward. For each tree, one scans through 
all the leaf nodes and 
collect information regarding if two points are in the same leaf node, and then aggregate such information from all trees in {\it rpForests}. For 
ease of description, let $S[A_1,A_2]$ denote all entries in matrix $S$ with their position indexed by the Cartesian product $A_1 \times A_2$ of 
two sets of integers $A_1$ and $A_2$. The generation of the similarity kernel is described as Algorithm~\ref{algorithm:rpfSimilarity}. Note that
here the notation, $\mathcal{N}$, for a node may also refer to the index of all points in this node for ease of description.
\begin{algorithm}
\caption{\it~~rpfSimilarity($\mathcal{F}$)} 
\label{algorithm:rpfSimilarity}
\begin{algorithmic}[1]
\STATE Initialize a similarity matrix $S \gets \bf{0}$; %%
\FOR {each tree $t\in \mathcal{F}$}
	\FOR {each leaf node $\mathcal{N}\in t$}
	\STATE Increase the similarity count for each entry in $S[\mathcal{N},\mathcal{N}]$; %%
	\ENDFOR
\ENDFOR%%
\STATE Set $S \gets S/(\mbox{number of trees in}~\mathcal{F})$; %%
\STATE Return $S$; %%
\end{algorithmic}
\end{algorithm} 
\\
\\
The similarity matrix $S$ as produced by Algorithm~\ref{algorithm:rpfSimilarity} is a valid kernel matrix. This can be argued 
as follows. Let $S^{(t)}$ denote the similarity matrix generated by tree $t$. Then $S^{(t)}$ is a block diagonal matrix 
with all points in the same leaf node form a diagonal block of all entries 1, and all off-diagonal blocks are 0 
as those correspond to points from different leaf nodes. Let matrix $M$ be one of the diagonal blocks in $S^{(t)}$.
Then $M$ is positive semidefinite, as the following holds
\begin{equation*}
z^T M z=(z_1+z_2+\cdots+z_m)^2 \geq 0
\end{equation*}
for any vector $z=(z_1,...,z_m)$. This implies that matrix $S^{(t)}$ is positive semidefinite. It follows that the average matrix 
\begin{equation*}
S=\frac{1}{T}\sum_{t=1}^T S^{(t)}
\end{equation*}
is also positive semidefinite. Thus the similarity matrix $S$ produced by {\it rpForests} is a valid kernel matrix. Due to its 
intimate connection to {\it rpForests}, the resulting kernel is termed as {\it rpf-kernel}.
\\
\\
On rpf-kernel $S$, it is straightforward that one can apply spectral clustering to obtain a clustering of the original data.  
This is described as Algorithm~\ref{algorithm:rpfCluster}.
\begin{algorithm}
\caption{\it~~rpfCluster($S$, K)} 
\label{algorithm:rpfCluster}
\begin{algorithmic}[1]
\STATE Threshold similarity kernel by $S_{ij} \gets 0$ if $S_{ij}<\beta_1$; %%
\STATE $S \gets exp(S/\beta_2)$ for some bandwidth $\beta_2$;
\STATE Apply spectral clustering to $S$ to get the cluster assignment; %%
\end{algorithmic}
\end{algorithm} 
%%\\
%%\\
Note that here we threshold the kernel $S$ following the same idea as \cite{CF}. This will help get rid of the spurious similarity 
between points in different clusters (in the ideal case, points from different clusters would have a 0 similarity if clustering
is concerned). Also similar as in the practice of using the Gaussian kernel in various kernel methods, we apply a bandwidth 
to reflect the correct scale at which the data are clustered.
\subsection{Spectral clustering}
\label{Sec:specCluster} 
In this subsection, we will briefly describe spectral clustering since it is used in the last stage of {\it rpfCluster} 
and also as a competing algorithm in our experiments (c.f. Section~\ref{section:competitors}). Spectral clustering works on the {\it affinity graph} 
$\mathcal{G} = (V, \mathcal{E})$ formed on the set of data points $V=\{X_1,...,X_n\}$, where each vertex corresponds to a data point and 
the edge, $S_{ij} \in \mathcal{E}$, encodes the affinity (similarity) between data points $X_i$ and $X_j$. The matrix $S=(S_{ij})_{i,j=1}^n$ 
is called the {\it affinity} or similarity matrix. 
\\
\\
There are several variants of spectral clustering~\cite{Ncut,NgJordan2002,KannanVempala2004}. In the present paper 
we adopt \emph{normalized cuts} (Ncut)~\cite{Ncut} for a concrete description (when the size of the dataset is big, e.g., 
larger than 2000, the Nystrom method \cite{NystromSpectral} is used to compute the eigen-decomposition); other formulations are 
also possible. A {\it graph cut} between two sets of vertices, $V_1, V_2 \subset V$, is the collection of all edges crossing 
$V_1$ and $V_2$, and the size of the cut is defined by 
\begin{equation*}
\mathcal{E}(V_1, V_2) = \sum_{X_i \in V_1, X_j \in V_2} S_{ij}.
\end{equation*}
\begin{figure}[htbp]
\centering
\begin{center}
\hspace{0cm}
\includegraphics[scale=0.34,clip]{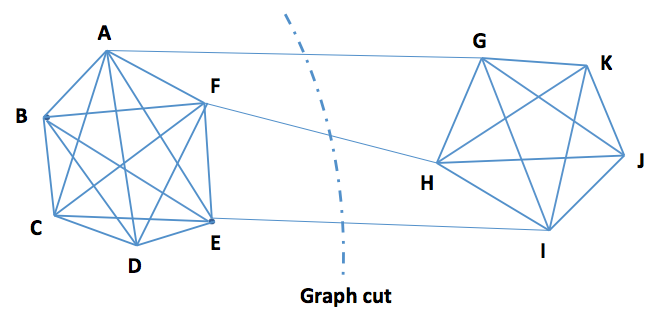}
\end{center}
%%\abovecaptionskip=-3pt
\caption{\it Illustration of a graph cut.} 
\label{figure:graphCut}
\end{figure}
Figure~\ref{figure:graphCut} is an illustration of the graph cut.
The dashed line cuts the graph into two partitions, $V_1=\{A,B,C,D,E,F\}$ and $V_2=\{G,H,I,J,K\}$. 
The total weights of three edges, $AG, FH, EI$, crossing $V_1$ and $V_2$ is the size of the cut. It 
is easy to see that the concept of graph cut can be extended to graph partitions that involve multiple
sets. Spectral clustering aims to find a minimal normalized graph cut
\begin{equation} \label{eq:ncut}
\arg\min_{V_1,...,V_K}\sum_{j=1}^K \frac{\mathcal{E}(V_j, V) - \mathcal{E}(V_j, V_j)}
{\mathcal{E}(V_j, V)},
\end{equation}
where $(V_1, \ldots, V_K)$ is a partition of $V$.
Directly solving \eqref{eq:ncut} is intractable as it is an integer programming problem, and spectral clustering is based 
on a relaxation of \eqref{eq:ncut} into an eigenvalue problem. In particular, Ncut relaxes the indicator vectors 
(corresponding to cluster membership) with real values, resulting in a generalized eigenvalue problem.
\\
\\
Ncut computes the second eigenvector of the Laplacian of matrix $S$ to find a bipartition
of the data. The nonnegative components of the eigenvector corresponds to one partition and the rest to the other. 
The same procedure is applied recursively until reaching the number of specified clusters.
\section{Theoretical analysis}
\label{section:theory}
The growth of {\it rpForests} involves quite a bit of randomness, i.e., in the choice of splitting directions and the splitting point. 
A central concern would be the quality of the rpf-kernel learned by {\it rpForests}---will the learned similarity kernel preserve
the true similarity values? Our analysis will ascertain that ``far-away" (or dissimilar) points would have low similarity while high 
similarity for nearby (or similar) points. The behavior of nearby points under {\it rpForests} was analyzed in \cite{rpForests2019} 
which states that the probability that such points are separated during the growth of {\it rpForests} is small thus one would expect 
a high similarity among such points. Here we follow their analysis but focus on the behavior of far-away points, and, for tractability, 
we only consider the basic implementation of {\it rpForests} \cite{rpForests2019}. Note that here {\it far-away} or {\it nearby} are not 
precisely defined; a crude rule for {\it nearby} is that one point is a k-nearest neighbor of another while {\it far-away} is when the 
distance between two points is larger than a small constant, say $\alpha$, and both $k$ and $\alpha$ are application dependent. 
Throughout our analysis, we use the Euclidean distance; however, it should be clear that other properly defined distance metrics 
also apply.  
\begin{figure}[htbp]
\centering
\begin{center}
\hspace{0cm}
\includegraphics[scale=0.28,clip]{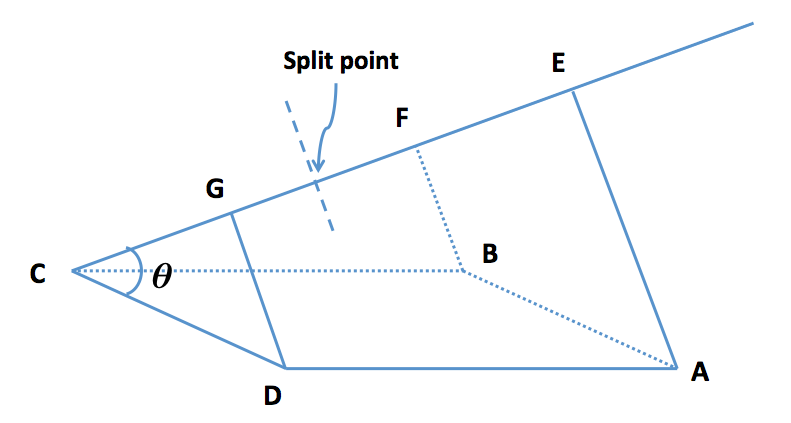}
\end{center}
%%\abovecaptionskip=-3pt
\caption{\it Illustration of the geometry involved in the random projection (image courtesy \cite{rpForests2019}). $A$ and $B$ 
are the two points of interest. Line $CE$ indicates the direction of random projection. $E$ and $F$ are the projection of points $A$ 
and $B$ onto line $CE$. Line $CD$ is parallel to $AB$, and $G$ is the projection of $D$ onto line $CE$. The split
point will lie anywhere on line $CE$ within the range of of the projections of all points at a tree node. } 
\label{figure:knnProof}
\end{figure}
\\
\\
We will use Figure~\ref{figure:knnProof} to assist our analysis.  
Assume $A$ and $B$ 
are the two points of interest, and let line $CE$ denote the direction of the random projection. We can conveniently assume 
that $C$ is the point of origin. Point $D$ is chosen such that $CD \parallel AB$ and that $|AB|=|CD|$ where $|.|$ denotes 
the length of a line segment and $\parallel$ indicates parallel. Assume $G, E, F$ are the points 
projected from $D, A, B$ onto $CE$, respectively. It was shown that $|CG|=|EF|$ in \cite{rpForests2019}. 
\\
\\
For a given point set $\Omega$, it can be cut along different directions. \cite{rpForests2019} considers a direction along which the 
data stretch the least (the size of such a stretch is called the {\it neck size}). Here we shall consider its counterpart which is defined 
as follows. 
\\
\\
\textbf{Definition.} Let $\Omega$ be a set of points. Define the size of its {\it principal stretch} as
\begin{equation*}
\rho(\Omega) = \sup_{\stackrel{\rightharpoonup}{r}} \sup_{x_1,x_2 \in \Omega} \{ \mid P_ {\stackrel{\rightharpoonup}{r}}(x_1) - P_{\stackrel{\rightharpoonup}{r}}(x_2) \mid\}.
\end{equation*} 
The neck size was used to bound the probability that nearby points stay unseparated 
under {\it rpForests}. In the following, we will state a result that characterizes the probability that far-away points would be separated under {\it rpForests}.
\begin{theorem}
\label{thm:singleSplit}
Let $\Omega$ be a set of points with a positive principal stretch, i.e., $\rho(\Omega)>0$.  Given any two points, $A, B \in \Omega$, for a random 
projection defined by the rpTree, 
\begin{equation*}
P(\text{A and B separated by a random projection} ) \geq \frac{2d}{\pi \rho(\Omega)},
\end{equation*}
where $d$ denote the distance between points $A$ and $B$.
\end{theorem}
\begin{proof}
The proof is based on the geometry illustrated in Figure~\ref{figure:knnProof}. 
The main idea is to condition on the angle, $\theta$, between the random direction $CE$ and line $AB$. Note that this is also 
the angle between the random direction and $CD$ as lines $AB \parallel CD$. It is clear that $0 \leq \theta \leq \pi/2$.
\\
\\
Given angle $\theta$, points $A$ and $B$ are separated if the splitting point is within their projections onto the random direction.
That is, the randomly selected splitting point falls between $E$ and $F$.
Let $L$ denote the size of the range of the projected points. That is, 
\begin{equation*}
L = \sup_{x_1,x_2 \in \Omega} \{ \mid P_ {\stackrel{\rightharpoonup}{r}}(x_1) - P_{\stackrel{\rightharpoonup}{r}}(x_2) \mid\}.
\end{equation*} 
As the splitting point is chose uniformly at random along the random direction, the probability that $A$ and $B$ are split
is given by $|EF|/L$. As $|EF|=|CG|=|AB| cos(\theta)$,
\begin{eqnarray*}
&& P(\text{A and B separated by a random projection} ) \\
&=& \int_{0}^{\pi/2}  P\left(\text{A and B are separated} \mid \theta \right) \cdot \frac{2}{\pi} d\theta \\
&=& \int_{0}^{\pi/2} \frac{|AB| cos(\theta)}{L} \cdot \frac{2} {\pi} d\theta \\
&\geq& \frac{2|AB|} {\pi \rho(\Omega)} \int_{0}^{\pi/2} cos(\theta)  d\theta 
= \frac{2d} {\pi \rho(\Omega)}. 
\end{eqnarray*}
\end{proof}
%%
%%
%%
%%
%%
%%
%%
%%
%%\\
%%\\
\noindent
Now assume a tree $t$ in {\it rpForests} splits $J_t \in [J_1, J_2]$ times, and 
the principal stretch of all the child nodes shrinks by a factor in the range of $[\gamma_1, \gamma_2] \subset (0,1)$. 
Then, for any two far-away points, $A$ and $B$, the probability that they would be separated in tree $t$ can be 
estimated as follows.
\\
\\
$A$ and $B$ are separated if they are separated at any one of the node splits. Let $N_t^{(1)}, N_t^{(2)}, \cdots, N_t^{(J_t)}$ 
indicate the collection of nodes in the growth of tree $t$ such that each node in the sequence is either the sibling or the 
descendent of proceeding nodes. Then 
\begin{eqnarray*}
&& P(\text{A and B separated in tree $t$}) \\
&=&  P\left(\text{$\cup_{i=1}^{J_t}$ (A and B separated when splitting node $N_t^{(i)}$} ) \right) \\
&\geq& P(\text{A and B separated when splitting node $N_t^{(J_t)}$}). 
\end{eqnarray*}
Now we can apply Theorem~\ref{thm:singleSplit} to node $N_t^{(J_t)}$, and get
\begin{equation}
P(\text{A and B separated in tree $t$}) \geq \frac{2d_{AB}} {\pi \rho(N_t^{(J_t)}) } 
\geq \frac{2d_{AB}} {\pi \rho\gamma_2^{J_1-1}}
\label{eq:lowerBoundSep}
\end{equation}
where $\rho$ is the size of the principal stretch over the entire data. One implication of \eqref{eq:lowerBoundSep} is that as 
the enclosing node shrinks, eventually two far-away points will be separated in probability. Under similar conditions, the result 
established in \cite{rpForests2019} implies that, for nearby points $A$ and $B$,
\begin{equation}
P(\text{A and B separated in tree $t$}) 
\leq \frac{2d_{AB}} {\pi \nu}  \frac{1}{\gamma_1^{J_2-2}(1-\gamma_1)}
\label{eq:upperBoundSep}
\end{equation}
where $\nu$ is the neck size of the entire data. 
%%\end{proof}
Bound \eqref{eq:upperBoundSep} indicates that the probability of separation for ``nearby" points remains small in tree $t$, since 
the near neighbor distance decreases very quickly for large set of data \cite{BickelYan2008,PenroseYukich2010}.  
Bounds \eqref{eq:lowerBoundSep} and \eqref{eq:upperBoundSep} will allow us to characterize important properties of the similarity
kernel learned by {\it rpForests}.
\\
\\
Now let us consider the similarity value for points $A$ and $B$ in the rpf-kernel produced by {\it rpForests} which consists of $T$ trees.
For $j=1,2,...,T$, define indicator random variables
\begin{equation*}
       \mathbb{I}_j~=\left\{
            \begin{array}{ll}
                1, & ~\mbox{if A and B separated in j-th tree}~\\
                0, & ~\mbox{otherwise.}
            \end{array}\right.
\end{equation*}
Then $\mathbb{I}_j$'s are independent and of the same distribution {\it conditional} on the data. Thus 
$\Delta=\mathbb{I}_1+\mathbb{I}_2+...+\mathbb{I}_T$ gives the number of trees for which $A$ and $B$ are separated. 
It is clear that $\Delta$ follows a binomial distribution or is approximately normal if $T$ is large. Assume $\mathbb{I}_j$ 
has a mean $\mu$ and standard deviation $\sigma$. For ``far-away" points $A$ and $B$, we 
wish to show that, with high probability, their similarity is smaller than a small value (not arbitrary small), say $\delta$. 
We have the following normal approximation
\begin{eqnarray}
&& P(\mbox{$A$ and $B$ have a similarity at most $\delta$}) \nonumber\\
&=& P(\Delta \geq (1-\delta) T) \nonumber\\
&\approx& P(Z \geq (1-\delta-\mu) \sqrt{T}/\sigma) \label{eq:probZ}
\end{eqnarray}
where $Z$ is the standard normal random variable. We wish to argue that $1-\delta-\mu<0$ for some small $\delta>0$. Then, 
if $T$ is large, the probability in \eqref{eq:probZ} will be high since then $(1-\delta-\mu) \sqrt{T}/\sigma$ will be very far towards $-\infty$.
By \eqref{eq:lowerBoundSep}, 
\begin{equation}
\mu = P(\mbox{$A$ and $B$ separated in a tree}) \geq 2d_{AB} /  (\pi \rho\gamma_2^{J_1-1}). \label{eq:uInequ}
\end{equation} 
If the number of data points $n$ increases, then the number of splits $J$ will also increase. As long as the principal stretch size of the 
nodes shrinks steadily, that is, $\gamma_2$ is strictly less than 1. Thus $\mu$ will be close to 1, and would satisfy $1-\delta-\mu<0$ for
far-away points $A$ and $B$.
Therefore, with high probability, the similarity of far-away points $A$ and $B$ will stay below some small $\delta$.
%%
%%\\
\\
\\
For nearby points $A$ and $B$, we can follow a similar argument, i.e., \eqref{eq:upperBoundSep}, and conclude that, with high probability, 
their similarity will be large. Thus, we have argued that the rpf-kernel produced by {\it rpForests} has the desirable property: {\it far-away 
(dissimilar) points have low similarity while high similarity for nearby (similar) points}. In Section~\ref{sec:expExamples}
we will provide a toy example to demonstrate such a property.
\section{Related work}
\label{section:related}
Similarity plays a very important role in machine learning. Due to the intimate connection between similarity and metric 
learning, we do not distinguish the two in this section. A simple similarity measure is typically defined through a closed-form 
function such as the cosine (or weighted cosine along principal directions \cite{ShahabiYan2003}), 
Euclidean (or Minkoski) distance function etc. More sophisticated is the bilinear 
similarity which measures the similarity of any two objects $x_1, x_2 \in \mathbb{R}^p$ by a bilinear form
\begin{equation*}
f_S(x_1,x_2) = (x_1-x_2)^T S (x_1-x_2)
\end{equation*}
where $S$ is a symmetric positive semidefinite matrix. This allows to incorporate feature weights or feature 
correlations into the similarity, and the Mahalanobis distance \cite{Anderson1958} is a classic example. Indeed many research in similarity learning
starts from an early work on learning the {\it Mahalanobis distance} with side information such as examples of similar pairs of 
objects \cite{XingNgJordanRussell2002}. There are many followup work along this line, for example, \cite{SchultzJoachims2003}
regularizes the learning by large margin, \cite{Bar-Hillel2005} uses side-information in the form of equivalence constraints, \cite{GoldbergerHRS2005} 
attempts to optimize leave-one-out error of a stochastic nearest neighbor classifier, \cite{DavisKJSD2007} minimizes the relative 
entropy between two Gaussians under constraints on the distance function, \cite{WeinbergerSaul2009} learns a large-margin nearest 
neighbor metric such that k-nearest neighbors are of the same class while different classes otherwise. Later work 
either adds more constraints (such as sparsity) or extends the Mahalanobis metric, or on broader classes of problems (such as ranking);
this includes \cite{ChechikBengio2010, Hirzer2012, LimLanckriet2014, LiuMJC2015, LiuBelletSha2015, KangPengCheng2017}. Similarity 
or metric learning is a big topic, and it is beyond our scope to have a more thorough review on existing work; readers can refer to
surveys \cite{YangJin2006,Kulis2012, BelletHabrardSebban2014, MoutafisLK2017} and references therein. Existing work are 
almost exclusively supervised or weakly supervised in nature, our work is different in that it is unsupervised.
\\
\\
Another related topic is clustering ensemble \cite{StrehlGhosh2002, GionisMannila2005, BertoniValentini2005, CF}. Clustering ensemble 
works by first generating many clustering instances, and then produce the final cluster by aggregating results from individual clustering instances. 
Literature on clustering ensemble is huge; readers can refer to \cite{StrehlGhosh2002, FernBrodley2003, CF, BoongoenIamon2018} and references 
therein. Clustering ensemble is related as each tree in {\it rpfCluster} can be viewed as an instance of clustering where points in the 
same leaf node form a cluster. Most closely related are random projection based clustering ensemble \cite{FernBrodley2003} and 
Cluster Forests (CF) \cite{CF}. \cite{FernBrodley2003} generates individual clustering instances by random projection of the original 
data onto some low dimensional space and then perform clustering. {\it rpfCluster} differs by generating a clustering instance through 
the growth of a tree which iteratively refines the clustering, or, {\it rpfCluster} projects the data onto low dimensional spaces by a series 
of random projections with each to a one-dimension space. Same as {\it rpfCluster}, CF also iteratively improves clustering instances, 
and produces the final cluster by spectral clustering on the learned similarity kernel; the difference is that CF generates clustering 
instances by a base clustering algorithm and refines each clustering instances by randomized feature pursuits.
\\
\\
Finally there are connections to Random Forests (RF) \cite{RF}. Both RF and {\it rpfCluster} generates ensemble of trees with random ingredients; 
the difference is that RF is {\it supervised} as tree growth is guided by class labels while {\it rpfCluster} grows trees {\it unsupervisedly},
also RF splits on coordinates while {\it rpfCluster} on random projections. 
Both {\it rpfCluster} and CF can be viewed as {\it unsupervised extensions} to RF. CF aims at clustering by ensemble of iteratively 
refined clustering instances through randomized feature pursuits (see also \cite{BickelNadler2018}) while {\it rpfCluster} learns the 
similarity kernel by ensemble of random 
projection trees. Another connection is that RF can run in unsupervised mode \cite{RF,ShiHorvath2006} to learn a suitable distance metric 
for clustering; this is done by synthesizing a contrast pattern through randomization of the original data by randomly permuting the data
along each of its features thus breaking the covariance structure of the data (implemented by the {\it proximity} option in the R package ``randomForest"), 
which is fundamentally different from our approach. 
\section{Experiments}
\label{section:evaluation}
Our experiments consists of two parts. In the first part, we will give illustrative examples to help the readers better appreciate 
the desired property of the rpf-kernel learned by {\it rpForests}, and its relevance for clustering. In the second part, 
we evaluate the empirical clustering performance of {\it rpfCluster} and compare it to three competing clustering algorithms, including
K-means clustering \cite{HartiganWong1979} as the baseline, the NJW algorithm \cite{NgJordan2002} as a popular implementation
of spectral clustering, and CF \cite{CF} as state-of-the-art ensemble clustering algorithm, on a wide variety of real 
datasets. The two parts are presented in Section~\ref{sec:expExamples} and 
Section~\ref{sec:expRealData}, respectively.
\subsection{Illustrative examples}
\label{sec:expExamples}
In this section, we will provide two illustrative examples. One serves to illustrate the desired property of similarity 
kernel learned by {\it rpForests}, and the other to demonstrate the propagation of similarity through 
points in the same cluster.
\subsubsection{Example on similarity kernel}
To appreciate the desirable property of the rpf-kernel produced by {\it rpForests}, we will use a popular yet simple 
dataset, the {\it Iris flower} data. It was introduced by one of the founders of modern statistics, 
R. A. Fisher, in 1936 \cite{Fisher1936} for discriminant analysis, and has since become one of the most widely used datasets. 
The data consist of three species of Iris, {\it Iris setosa}, {\it Iris versicolor}, and {\it Iris virginica}, with 50 instances each 
on four features, the length and width of the sepals and petals, respectively.
\\
\begin{figure}[htbp]
\centering
\begin{center}
\hspace{0cm}
\includegraphics[scale=0.35,clip]{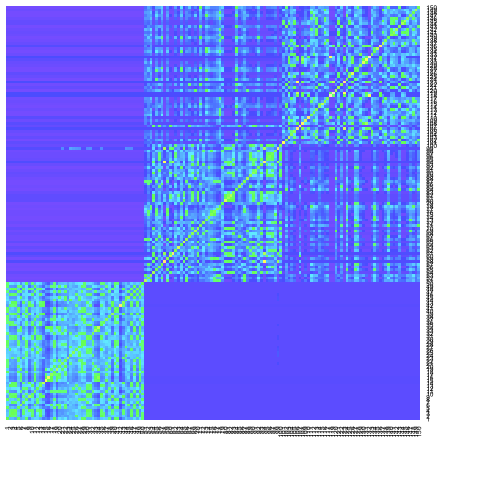}
\includegraphics[scale=0.35,clip]{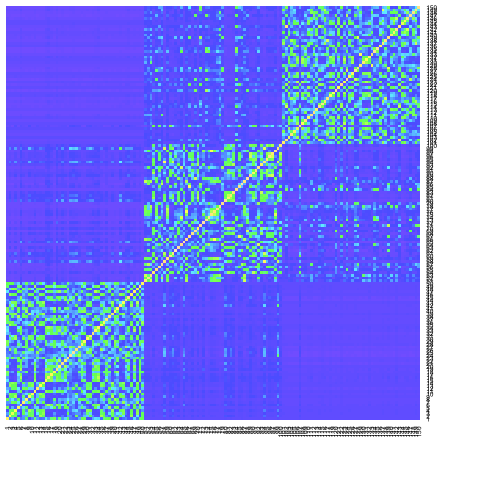}
\end{center}
\abovecaptionskip=-10pt
\caption{\it Heatmap of the similarity matrix generated by the Gaussian kernel (left) and by {\it rpForests} (right), respectively. } 
\label{figure:simIris}
\end{figure}
\\
\\
Figure~\ref{figure:simIris} shows the heatmap of the similarity matrix generated by the Gaussian kernel and by {\it rpForests}, 
respectively. 
It can be seen that, in both cases, the similarity between data points in the same species (i.e., the diagonal blocks) are higher 
than otherwise. In particular, the similarities are close to 0 between the {\it Iris setosa} and the {\it Iris virginica}. This is expected. 
However, the contrast between other diagonal and non-diagonal blocks by {\it rpForests} is much sharper than those by the Gaussian 
kernel. We attribute this to the Gaussian kernel as a sole function of the distance between points which is Euclidean and every 
feature is equally weighted, and further the potential smoothing effect of the Gaussian kernel. In contrast, the kernel 
formed by {\it rpForests} would be a result of both the distance between points and their neighborhood thus is able to take advantage 
of the {\it geometry} in the data. Further work is being conducted to understand these.
\\
\\
We then run spectral clustering algorithm on the rpf-kernel generated
by {\it rpForests}, and obtain a clustering and co-cluster accuracy (see definition in Section~\ref{section:metrics}) of 96.67\% and  
94.95\%, respectively. This is at the level of {\it classification} by some best classifiers 
such as RF despite that {\it rpfCluster} is unsupervised learning. The clustering and co-cluster accuracy 
on the Gaussian kernel are noticeably inferior, which are 92.00\% and 90.55\%, respectively.
%%\\
%%\\
\subsubsection{Example on clustering}
\label{sec:numeric}
%%\linespread{1.2}
When adopting the rpf-kernel produced by {\it rpForests} for clustering, one might notice that two points in the same cluster may 
have a low similarity (even though rpf-kernel is able to pick up structural information from the data) if they are far away from each 
other (since these two points are far away, likely they would be put into different
leaf nodes by {\it rpForests} thus a low value of similarity in the resulting similarity kernel); for example, two points that are located at 
the opposite ends of the cluster. However, this will {\it not} cause a problem for the subsequent spectral clustering which is built on local 
similarity. One empirical evidence is various algorithms for the speeding up of spectral clustering by sparsifying the similarity matrix based 
on k-nearest neighbors where similarity between non-kNNs are truncated to be 0. 
\\
\\
Here we supply a simple numerical example for illustration: far-away points will not be assigned to different clusters by spectral clustering 
as long as they are inside a region where all points has high similarity to their near neighbors. Assume there are 9 points, $X_{1-9}$, on a line 
which form two clusters, $X_{1-4}$ and $X_{5-9}$, respectively. Assume, for each point, only its immediate neighbors have a non-zero similarity; 
further assume the similarity between $X_4$ and $X_5$ (from different clusters) is 0.3. The similarity matrix is given by
%%\small
\[
%%\begin{equation*}
\small
A = \left[
\begin{array}{cccc;{2pt/2pt}ccccc}
%%\begin{array}{cccc;ccccc}
1.0  &0.9  &0.0  &0.0  &0.0  &0.0  &0.0  &0.0  &0.0\\
0.9  &1.0  &0.9  &0.0  &0.0  &0.0  &0.0  &0.0  &0.0\\
0.0  &0.9  &1.0  &0.9  &0.0  &0.0  &0.0  &0.0  &0.0\\
0.0  &0.0  &0.9  &1.0  &0.3  &0.0  &0.0  &0.0  &0.0\\ \hdashline[2pt/2pt]
0.0  &0.0  &0.0  &0.3  &1.0  &0.9  &0.0  &0.0  &0.0\\
0.0  &0.0  &0.0  &0.0  &0.9  &1.0  &0.9  &0.0  &0.0\\
0.0  &0.0  &0.0  &0.0  &0.0  &0.9  &1.0  &0.9  &0.0\\
0.0  &0.0  &0.0  &0.0  &0.0  &0.0  &0.9  &1.0  &0.9\\
0.0  &0.0  &0.0  &0.0  &0.0  &0.0  &0.0  &0.9  &1.0
\end{array} \right]
\normalsize
%%\end{equation*}
\] 
%%\normalsize
The eigenvector used for spectral clustering (normalized cuts \cite{Ncut}) is given by
\small
\begin{equation*}
\label{eq:2ndEigenvector} [0.340~ 0.469~
0.392~ 0.219~ -0.098~ -0.249~ -0.351~
-0.419~ -0.304]^T,
\end{equation*}
\normalsize
which gives the expected clustering, i.e., positive components correspond to one cluster and the rest the other cluster.
\subsection{Experiments on real datasets}
\label{sec:expRealData}
A wide range of real data are used for performance assessment. This includes 11 benchmark datasets taken from the UC Irvine 
Machine Learning Repository \cite{UCI}, namely, Soybean, SPECT Heart, image segmentation (ImgSeg), Heart, Wine, 
Wisconsin breast cancer (WDBC), robot execution failure (lp5), Madelon, Musk, Naval Plants, and the Magic 
Gamma (mGamma) dataset, as well as a remote sensing dataset (RS) \cite{rsDiagnosis2019}, totally 12 datasets. For WDBC, 
it is standardized on its \{5, 6, 25, 26\}-th features; the Wine dataset is standardized on its \{6, 14\}-th features; for the Naval 
Plants data, we treat any record of measurements as requiring maintenance if both the q3Compressor and q3Turbine variables 
are above their median values thus converting the original 
numerical values into categorical labels. A summary of the datasets is given in Table~\ref{table:datasets}. Note that all datasets come with 
labels. We made such a choice by recognizing that the ultimate goal of clustering is to get the membership of all the points right; 
many existing metrics for evaluating clustering algorithms are often a surrogate of this due to the lack of true labels. 
We will compare the performance by {\it rpfCluster} and three competing algorithms on two different performance metrics.
\begin{table}[htbp]
\begin{center}
%%\begin{small}
%%\begin{minipage}[b]{0.5\linewidth}%%\centering
\begin{tabular}{rrrr}
\hline
%%\toprule
Dataset          & Features     &  Classes   &  \#Instances\\
\hline
%%\midrule
Soybean          & 35              & 4             & 47\\
SPECT            & 22              & 2             & 267\\
ImgSeg           & 19              & 7             & 2100\\
Heart            & 13              & 2             & 270\\
Wine             & 13              & 3             & 178\\
WDBC             & 30              & 2             & 569\\
Robot            & 90              & 5             & 164\\
Madelon          & 500             & 2             & 2000 \\
RS			&56				&7		&3303\\
Musk		&166		 &2			&6598\\
NavalPlants     &16			&2			&11934\\
mGamma	    &10			&2			&19020\\
\hline
%%\bottomrule
\end{tabular}
%%\end{minipage}
\end{center}
%%\vskip -0.2in
\caption{\it A summary of datasets.} \label{table:datasets}
\end{table}
\subsubsection{Performance metrics}
\label{section:metrics}
Two different performance metrics are used; these are adopted from \cite{CF}. One is the clustering accuracy, 
and the other is the co-cluster accuracy. Having different performance metrics allows to assess a clustering 
algorithm from different perspectives since one metric may favor certain aspects while overlooking others. Using the 
clustering or co-cluster accuracy has the advantage of closely aligning to the ultimate goal of clustering---assigning
data points to proper groups---while other metrics are often a surrogate of the cluster membership.
In the following, we formally define the two performance metrics.
\\
\\
\textbf{Definition.} Let $\mathcal{L}=\{1,2,...,l\}$ denote the set of class labels, and
$h(.)$ and $\hat{h}(.)$ the true label and the label obtained by a
clustering algorithm, respectively. The {\it clustering accuracy} is
defined as
\begin{equation}
\label{clusterAccuracy} \rho_c(\hat{h})=\max_{\tau \in \Pi_{\mathcal{L}}}
\left\{\frac{1}{n}\sum_{i=1}^n \mathbb{I}\{\tau
\left(h(X_i)\right)=\hat{h}(X_i)\}\right\},
\end{equation}
where $\mathbb{I}$ is the indicator function and $\Pi_{\mathcal{L}}$ is the set of all permutations on the label 
set $\mathcal{L}$. It measures the fraction of labels given by a clustering algorithm that agree with the true labels 
(or labels come with the dataset; we call these generally as reference labels for simplicity of description) up to a 
permutation of the labels. This is a natural extension of the classification 
accuracy (under 0-1 loss) and has been used by many work in clustering 
\cite{XingNgJordanRussell2002,MeilaShortreed2005, YanHuangJordan2009}.
\\
\\
\textbf{Definition.} The {\it co-cluster accuracy} is defined by
\begin{equation*}
\rho_r=\frac{\mbox{Number of correctly clustered pairs}}{\mbox{Total
number of pairs}} 
\end{equation*}
where by {\it correctly clustered pair} we mean two data points, determined to be in the same cluster by a clustering 
algorithm, are also in the same cluster according to the reference labels. Both the clustering accuracy and the co-cluster
accuracy are converted to a scale of 100\%.
%%\\
%%
%%
%%
\subsubsection{Competing methods}
\label{section:competitors}
We compare {\it rpfCluster} to three other clustering algorithms---K-means clustering, the NJW algorithm, and CF. Among 
these, K-means clustering is one of the most widely used clustering algorithms, NJW is a popular implementation of spectral 
clustering (commonly acknowledged as the class of best clustering algorithms), and CF is a state-of-the-art clustering ensemble 
algorithm. The comparison can also be viewed on different similarity kernels. The similarity kernel in CF is learned by randomized 
feature pursuits in individual clustering instances and then aggregate. In NJW, the Gaussian kernel is used by default.
For the rest of this section, we will briefly describe K-means clustering, NJW and 
CF followed by a description on their implementations and parameters. 
\\
\\
$K$-means clustering \cite{HartiganWong1979, lloyd1982} seeks to find a partition of the data into $K$ subsets 
$S_1, S_2, ..., S_K$ such that the within-cluster sum of squares is minimized.
Directly solving the problem 
is NP-hard. It is often implemented by an iteration of a two-step procedure. 
Starting with a set of randomly chosen initial cluster centroids, it iterates between: 1) assigning all the data points to their closest 
centroid (data points associated with the same centroid forms a cluster); and 2) recalculating the cluster centroid for data points 
in the same cluster, until the changes in the within-cluster sum of squares is small. Such a procedure would converge to local
optima, and repeating this procedure for a number of times typically leads to fairly satisfactory results. K-means clustering is
widely used in practice, as it is simple to implement and computationally fast.
\\
\\
The NJW algorithm \cite{NgJordan2002} is a popular variant of spectral clustering. It works on the eigen-decomposition 
of the Laplacian of the similarity matrix over the data; the Gaussian kernel is used in the similarity measure which is defined as 
\begin{equation*}
S_{ij}=exp(-||X_i-X_j||^2/(2\sigma^2)),
\end{equation*}
where $\sigma$ is the bandwidth. Then it embeds the original data points into 
a space spanned by the top few eigenvectors, followed by a normalization of the resulting embedding to the unit sphere. 
K-means clustering is then performed to get the cluster membership assignment.
\\
\\
As mentioned in Section~\ref{section:related}, CF \cite{CF} is a clustering ensemble algorithm. Each clustering instance in 
CF is generated by randomized feature pursuit according to the $\kappa$ criterion \cite{CF} which is the ratio of the within-cluster 
and between-cluster sum of squared distances. It starts with a randomly selected features, then generates several sets of candidate 
features to run by a base clustering algorithm (e.g., K-means clustering as implemented in \cite{CF}) and keep the set of features 
that lead to the maximum decreases in the $\kappa$ value. This procedure is repeated until changes to the $\kappa$ value become 
too small. On each clustering instance a co-cluster indicator matrix is calculated, and the indicator matrices are then averaged to get 
the similarity matrix. Then similar steps as {\it rpfCluster} are executed to get the final clustering assignment. For details about its 
parameter settings, please refer to \cite{CF} from which we also adopt some experimental results in our comparison.
\\
\\
For $K$-means clustering, the R package {\it kmeans()} was used with the ``Hartigan-Wong" \cite{HartiganWong1979} initialization, 
and the two parameters $(n_{it}, n_{rst})$, which stands for the maximum number of iterations and the number of restarts during each 
run, respectively, are set to be $(1000,100)$. For NJW, function {\it specc()} of the R package ``kernlab" 
\cite{kernlab} was used with the Gaussian kernel and an automatic search of the local bandwidth parameters. 
The number of trees in {\it rpForests} are chosen from $\{200,400,600\}$ and the difference in results is very small, the node 
splitting constant $n_s$ is $30$ except for $12$ for Soybean and $200$ for Madelon, the threshold level $\beta_1$ 
is chosen from $\{0, 0.1, 0.2, 0.3, 0.4\}$, and the step size for the search of bandwidth $\beta_2$ is 0.01 within (0,1] while 0.1 over (1,200]. 
\subsubsection{Experimental results}
\label{section:expUCI} 
The evaluation is based on two different performance 
metrics for, $\rho_c$ and $\rho_r$, as defined in Section~\ref{section:metrics}.
The results under $\rho_c$ and $\rho_r$ are shown as Figure~\ref{figure:comparisonsC} 
and Figure~\ref{figure:comparisonsR}, respectively. 
\begin{figure}[htp]
\centering
\begin{center}
\hspace{0cm}
\includegraphics[scale=0.5,clip]{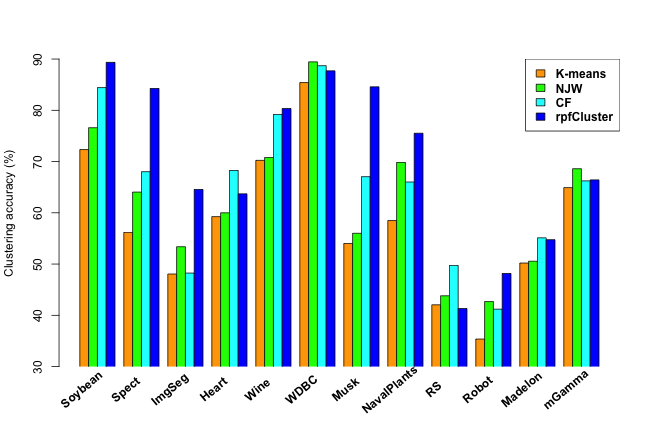}
\end{center}
\abovecaptionskip=-15pt
\caption{\it Comparison between K-means clustering, spectral clustering (NJW),
CF, and rpfCluster for clustering accuracy $\rho_c$. } 
\label{figure:comparisonsC}
\end{figure}
\begin{figure}[htp]
\centering
\begin{center}
\hspace{0cm}
\includegraphics[scale=0.5,clip]{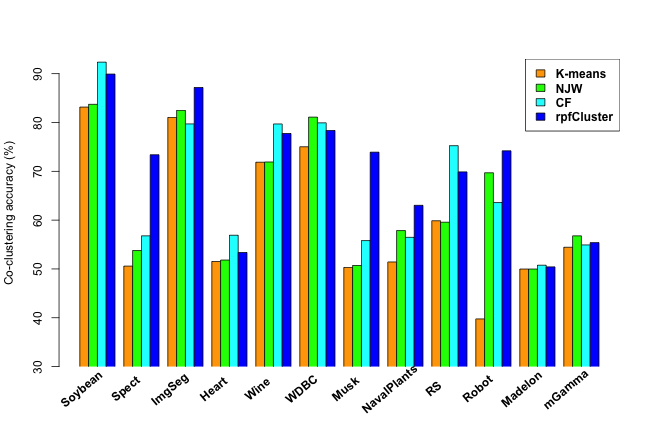}
\end{center}
\abovecaptionskip=-15pt
\caption{\it Comparison between K-means clustering, spectral clustering (NJW),
CF, and rpfCluster for co-clustering accuracy $\rho_r$. } 
\label{figure:comparisonsR}
\end{figure}
\\
\\
On all but two of the 12 datasets, either {\it rpfCluster} or CF is leading
with {\it rpfCluster} having an edge. Under clustering accuracy $\rho_c$, {\it rpfCluster} is leading on 7 datasets, and ranks 
the second on 3 datasets while CF leads on 3 and seconds on 5. For co-cluster accuracy $\rho_r$, {\it rpfCluster} leads on 
5 datasets and seconds on 6 while CF leads on 5 and seconds on 3. Overall, {\it rpfCluster} outperforms CF (also NJW
and K-means clustering) on both two performance metrics. 
\section{Conclusions}
\label{section:conclusion}
We have proposed an effective approach for the unsupervised learning of a similarity kernel by {\it rpForests}. Our approach 
combines the power of ensemble methodology and the flexibility of trees. It is simple to implement, and readily adapt to the 
geometry of the underlying data. Our theoretical analysis reveals highly desirable property of the learned rpf-kernel: {\it far-away points 
have low similarity while high similarity for nearby points}, and the similarities have a native interpretation as the probability of 
points staying in the tree leaf nodes during the growth of {\it rpForests}. The learned rpf-kernel is readily incorporated into 
our clustering algorithm {\it rpfCluster}. On a wide variety of real and benchmark datasets, {\it rpfCluster} compares favorably to 
spectral clustering and a state-of-the-art 
clustering ensemble algorithm. Given the desirable theoretical property and the highly competitive empirical 
performance on clustering, we expect rpf-kernel to be applicable to other problems of an unsupervised nature or as a 
regularizer in supervised or weakly supervised settings. As each projection during the tree growth can be viewed as a coordinate, 
effectively {\it rpForests} is exploring some low dimensional manifolds (of a dimension approximately $O(\log(n))$) in high dimensional data, 
thus we expect {\it rpForests} to be also useful in manifold learning \cite{Cayton2008,HuoNiSmith2007}.       
%%
%%\bibliographystyle{plain}
%%\bibliography{../../../myBib}

\begin{thebibliography}{10}

\bibitem{Anderson1958}
T.~W. Anderson.
\newblock {\em An Introduction to Multivariate Statistical Analysis}.
\newblock John Wiley \& Sons, 1958.

\bibitem{AngiulliPizzu2002}
F.~Angiulli and C.~Pizzu.
\newblock Fast outlier detection in high dimensional spaces.
\newblock {\em Lecture Notes in Computer Science}, 2431:43--78, 2002.

\bibitem{BachJordan2003}
F.~Bach and M.~I. Jordan.
\newblock Kernel independent component analysis.
\newblock {\em Journal of Machine Learning Research}, 3:1--48, 2003.

\bibitem{Bar-Hillel2005}
A.~Bar-Hillel, T.~Hertz, N.~Shental, and D.~Weinshall.
\newblock Learning a {M}ahalanobis metric from equivalence constraints.
\newblock {\em Journal of Machine Learning Research}, 6:937--965, 2005.

\bibitem{BelkinNiyogiSindhwani2006}
M.~Belkin, P.~Niyogi, and V.~Sindhwani.
\newblock Manifold regularization: A geometric framework for learning from
  labeled and unlabeled examples.
\newblock {\em Journal of Machine Learning Research}, 7:2399--2434, 2006.

\bibitem{BelletHabrardSebban2014}
A.~Bellet, A.~Habrard, and M.~Sebban.
\newblock A survey on metric learning for feature vectors and structured data.
\newblock {\em arXiv:1306.6709.v4}, 2014.

\bibitem{Bentley1975}
J.~Bentley.
\newblock Multidimensional binary search trees used for associative searching.
\newblock {\em Communications of the ACM}, 18(9):509--517, 1975.

\bibitem{BertoniValentini2005}
A.~Bertoni and G.~Valentini.
\newblock Ensembles based on random projections to improve the accuracy of
  clustering algorithms.
\newblock In {\em Proceedings of the 16th Italian conference on Neural Nets},
  pages 31--37, 2005.

\bibitem{BickelBreiman1983}
P.~J. Bickel and L.~Breiman.
\newblock Sums of functions of nearest neighbor distances, moment bounds, limit
  theorems and a goodness of fit test.
\newblock {\em The Annals of Probability}, 11(1):185--214, 1983.

\bibitem{BickelNadler2018}
P.~J. Bickel, G.~Kur, and B.~Nadler.
\newblock Projection pursuit in high dimensions.
\newblock {\em Proceedings of the National Academy of Sciences, U. S. A.},
  115(37):9151--9156, 2018.

\bibitem{BickelYan2008}
P.~J. Bickel and D.~Yan.
\newblock Sparsity and the possibility of inference.
\newblock {\em Sankhya: The Indian Journal of Statistics, Series A (2008-)},
  70(1):1--24, 2008.

\bibitem{BoongoenIamon2018}
T.~Boongoen and N.~Iam-On.
\newblock Cluster ensembles: A survey of approaches with recent extensions and
  applications.
\newblock {\em Computer Science Review}, 28:1--25, 2018.

\bibitem{Bagging}
L.~Breiman.
\newblock Bagging predicators.
\newblock {\em Machine Learning}, 24(2):123--140, 1996.

\bibitem{RF}
L.~Breiman.
\newblock Random {F}orests.
\newblock {\em Machine Learning}, 45(1):5--32, 2001.

\bibitem{Cayton2008}
L.~Cayton.
\newblock Algorithms for manifold learning.
\newblock {\em Technical Report CS2008-0923, Department of Computer Science, UC
  San Diego}, 2008.

\bibitem{ChandolaBanerjeeKumar2007}
V.~Chandola, A.~Banerjee, and V.~Kumar.
\newblock Anomaly detection: A survey.
\newblock {\em Technical Report, University of Minnesota}, 2007.

\bibitem{ChapelleWeston2003}
O.~Chapelle, J.~Weston, and B.~Sch\"{o}lkopf.
\newblock Cluster kernels for semi-supervised learning.
\newblock In {\em Advances in Neural Information Processing Systems 15}, pages
  601--608, 2003.

\bibitem{ChechikBengio2010}
G.~Chechik, V.~Sharma, U.~Shalit, and S.~Bengio.
\newblock Large scale online learning of image similarity through ranking.
\newblock {\em Journal of Machine Learning Research}, 11:1109--1135, 2010.

\bibitem{CortesVapnik1995}
C.~Cortes and V.~N. Vapnik.
\newblock Support-vector networks.
\newblock {\em Machine Learning}, 20(3):273--297, 1995.

\bibitem{RPTree}
S.~Dasgupta and Y.~Freund.
\newblock Random projection trees and low dimensional manifolds.
\newblock In {\em Fortieth {ACM} Symposium on Theory of Computing (STOC)},
  2008.

\bibitem{DasguptaSinha2015}
S.~Dasgupta and K.~Sinha.
\newblock Randomized partition trees for nearest neighbor search.
\newblock {\em Journal Algorithmica}, 72(1):237--263, 2015.

\bibitem{DavisKJSD2007}
J.~Davis, B.~Kulis, P.~Jain, S.~Sra, and I.~S. Dhillon.
\newblock Information-theoretic metric learning.
\newblock In {\em Proceedings of the 24th International Conference on Machine
  Learning}, pages 209--216, 2007.

\bibitem{DhillonGK2004}
I.~Dhillon, Y.~Guan, and B.~Kulis.
\newblock Kernel k-means: spectral clustering and normalized cuts.
\newblock In {\em Proceedings of the tenth ACM international conference on
  Knowledge discovery and data mining (SIGKDD)}, 2004.

\bibitem{DingShaoFu2018}
Z.~Ding, M.~Shao, and Y.~Fu.
\newblock Robust multi-view representation: A unified perspective from
  multi-view learning to domain adaption.
\newblock In {\em Proceedings of 27th International Joint Conference on
  Artificial Intelligence}, pages 5434--5440, 2018.

\bibitem{DongMosesLi2011}
W.~Dong, C.~Moses, and K.~Li.
\newblock Efficient k-nearest neighbor graph construction for generic
  similarity measures.
\newblock In {\em Proceedings of the 20th International Conference on World
  Wide Web}, 2011.

\bibitem{FernBrodley2003}
X.~Z. Fern and C.~E. Brodley.
\newblock Random projection for high dimensional data clustering: A cluster
  ensemble approach.
\newblock In {\em Proceedings of the 20th International Conference on Machine
  Learning (ICML)}, 2003.

\bibitem{netflix2012}
A.~Feuerverger, Y.~He, and S.~Khatri.
\newblock Statistical significance of the netflix challenge.
\newblock {\em Statistical Science}, 27(2):202--231, 2012.

\bibitem{Fisher1936}
R.~A. Fisher.
\newblock The use of multiple measurements in taxonomic problems.
\newblock {\em Annals of Eugenics}, 7(2):179--188, 1936.

\bibitem{NystromSpectral}
C.~Fowlkes, S.~Belongie, F.~Chung, and J.~Malik.
\newblock Spectral grouping using the {N}ystr${\ddot{\text{o}}}$m method.
\newblock {\em IEEE Transactions on Pattern Analysis and Machine Intelligence},
  26(2):214--225, 2004.

\bibitem{Adaboost}
Y.~Freund and R.~Schapire.
\newblock Experiments with a new boosting algorithm.
\newblock In {\em International Conference on Machine Learning (ICML)}, 1996.

\bibitem{FriedmanBentleyFinkel1977}
J.~Friedman, J.~Bentley, and R.~Finkel.
\newblock An algorithm for finding the best matches in logarithmic expected
  time.
\newblock {\em ACM Transactions on Mathematical Software}, 3(3):209--226, 1977.

\bibitem{FukumizuBachGretton2007}
K.~Fukumizu, F.~Bach, and A.~Gretton.
\newblock Statistical consistency of kernel canonical correlation analysis.
\newblock {\em Journal of Machine Learning Research}, 8:361--383, 2007.

\bibitem{GionisMannila2005}
A.~Gionis, H.~Mannila, and P.~Tsaparas.
\newblock Cluster aggregation.
\newblock In {\em the 21st International Conference on Data Engineering
  (ICDE)}, 2005.

\bibitem{GoldbergerHRS2005}
J.~Goldberger, G.~E.~Hinton E, S.~Roweis, and R.~Salakhutdinov.
\newblock Neighbourhood components analysis.
\newblock In {\em Advances in Neural Information Processing Systems 17}, pages
  513--520, 2005.

\bibitem{HartiganWong1979}
J.~A. Hartigan and M.~A. Wong.
\newblock A {K}-means clustering algorithm.
\newblock {\em Applied Statistics}, 28(1):100--108, 1979.

\bibitem{Hirzer2012}
M.~Hirzer.
\newblock Large scale metric learning from equivalence constraints.
\newblock In {\em Proceedings of the 2012 IEEE Conference on Computer Vision
  and Pattern Recognition (CVPR)}, pages 2288--2295, 2012.

\bibitem{HofmannScholkopfSmola2008}
T.~Hofmann, B.~Sch${\ddot{\text{o}}}$lkopf, and A.~Smola.
\newblock Kernel methods in machine learning.
\newblock {\em The Annals of Statistics}, 36(3):1171--1220, 2008.

\bibitem{HuoNiSmith2007}
X.~Huo, X.~Ni, and A.~Smith.
\newblock A survey of manifold-based learning methods.
\newblock {\em Recent Advances in Data Mining of Enterprise Data}, pages
  691--745, 2007.

\bibitem{KangPengCheng2017}
Z.~Kang, C.~Peng, and Q.~Cheng.
\newblock Kernel-driven similarity learning.
\newblock {\em Neurocomputing}, 267(C):210--219, 2017.

\bibitem{KannanVempala2004}
R.~Kannan, S.~Vempala, and A.~Vetta.
\newblock On clusterings: Good, bad and spectral.
\newblock {\em Journal of the ACM}, 51(3):497--515, 2004.

\bibitem{kernlab}
A.~Karatzoglou, A.~Smola, and K.~Hornik.
\newblock {kernlab: Kernel-based Machine Learning Lab}.
\newblock {http://cran.r-project.org/web/packages/kernlab/index.html}, 2013.

\bibitem{Kleinbort2015EfficientHM}
M.~Kleinbort, O.~Salzman, and D.~Halperin.
\newblock Efficient high-quality motion planning by fast all-pairs
  r-nearest-neighbors.
\newblock In {\em IEEE International Conference on Robotics and Automation
  (ICRA)}, pages 2985--2990, 2015.

\bibitem{Kulis2012}
B.~Kulis.
\newblock {Metric Learning: A Survey}.
\newblock {\em Foundations and Trends in Machine Learning}, 5(4):287--364,
  2015.

\bibitem{LevinaBickel2005}
E.~Levina and P.~J. Bickel.
\newblock Maximum likelihood estimation of intrinsic dimension.
\newblock In {\em Advances in Neural Information Processing Systems 17}, 2005.

\bibitem{UCI}
M.~Lichman.
\newblock {UC Irvine Machine Learning Repository}.
\newblock {http://archive.ics.uci.edu/ml}, 2013.

\bibitem{LimLanckriet2014}
D.~Lim and G.~Lanckriet.
\newblock Efficient learning of mahalanobis metrics for ranking.
\newblock In {\em Proceedings of the 13rd International Conference on Machine
  Learning (ICML)}, 2014.

\bibitem{LiuBelletSha2015}
K.~Liu, A.~Bellet, and F.~Sha.
\newblock Similarity learning for high-dimensional sparse data.
\newblock In {\em { International Conference on Artificial Intelligence and
  Statistics (AISTATS)}}, 2015.

\bibitem{LiuMooreGray2004}
T.~Liu, A.~Moore, A.~Gray, and K.~Yang.
\newblock An investigation of practical approximate nearest neighbor
  algorithms.
\newblock In {\em Neural Information Processing Systems (NIPS)}, volume~19,
  pages 825--832, 2004.

\bibitem{LiuMJC2015}
W.~Liu, C.~Mu, R.~Ji, S.~Ma, J.~R. Smith, and S.-F. Chang.
\newblock Low-rank similarity metric learning in high dimensions.
\newblock In {\em Proceedings of the 29-th AAAI Conference on Artificial
  Intelligence}, pages 2792--2799, 2015.

\bibitem{lloyd1982}
S.~P. Lloyd.
\newblock Least squares quantization in {PCM}.
\newblock {\em IEEE Transactions on Information Theory}, 28(1):128--137, 1982.

\bibitem{MackRosenblatt1979}
Y.~P. Mack and M.~Rosenblatt.
\newblock Multivariate k-nearest neighbor density estimates.
\newblock {\em Journal of Multivariate Analysis}, 9:1--15, 1979.

\bibitem{MeilaShortreed2005}
M.~Meila, S.~Shortreed, and L.~Xu.
\newblock Regularized spectral learning.
\newblock Technical report, Department of Statistics, University of Washington,
  2005.

\bibitem{MoutafisLK2017}
P.~Moutafis, M.~Leng, and I.~A. Kakadiaris.
\newblock An overview and empirical comparison of distance metric learning
  methods.
\newblock {\em IEEE Transactions on Cybernetics}, 47(3):612--625, 2017.

\bibitem{NgJordan2002}
A.~Y. Ng, M.~I. Jordan, and Y.~Weiss.
\newblock On spectral clustering: analysis and an algorithm.
\newblock In {\em Neural Information Processing Systems (NIPS)}, volume~14,
  2002.

\bibitem{Otto2017ClusteringMOArxiv}
C.~Otto, D.~Wang, and A.~K. Jain.
\newblock {Clustering millions of faces by identity}.
\newblock {\em arXiv:1604.00989}, 2016.

\bibitem{PenroseYukich2010}
M.~Penrose and J.~Yukich.
\newblock Laws of large numbers and nearest neighbor distances.
\newblock {\em Advances in Directional and Linear Statistics}, pages 189--199,
  2010.

\bibitem{RamaswamyRastogiShim2000}
S.~Ramaswamy, R.~Rastogi, and K.~Shim.
\newblock Efficient algorithms for mining outliers from large data sets.
\newblock In {\em Proceedings of the ACM SIGMOD}, pages 427--438, 2000.

\bibitem{Scholkopf1998}
B.~Sch${\ddot{\text{o}}}$lkopf.
\newblock Nonlinear component analysis as a kernel eigenvalue problem.
\newblock {\em Neural Computation}, 10:1299--1319, 1998.

\bibitem{ScholkopfSmola2001}
B.~Sch${\ddot{\text{o}}}$lkopf and A.~Smola.
\newblock {\em Learning with kernels: Support Vector Machines, Regularization,
  Optimization, and Beyond}.
\newblock MIT Press, 2001.

\bibitem{SchultzJoachims2003}
M.~Schultz and T.~Joachims.
\newblock Learning a distance metric from relative comparisons.
\newblock In {\em Advances in Neural Information Processing Systems (NIPS)},
  2003.

\bibitem{ShahabiYan2003}
C.~Shahabi and D.~Yan.
\newblock Real-time pattern isolation and recognition over immersive sensor
  data streams.
\newblock In {\em Proceedings of the 9th International conference on
  multi-media modeling}, pages 93--113, 2003.

\bibitem{Ncut}
J.~Shi and J.~Malik.
\newblock Normalized cuts and image segmentation.
\newblock {\em IEEE Transactions on Pattern Analysis and Machine Intelligence},
  22(8):888--905, 2000.

\bibitem{ShiHorvath2006}
T.~Shi and S.~Horvath.
\newblock Unsupervised learning with random forest predictors.
\newblock {\em Journal of Computational and Graphical Statistics},
  15(1):118--138, 2006.

\bibitem{StrehlGhosh2002}
A.~Strehl and J.~Ghosh.
\newblock Cluster ensembles -- a knowlwdge reuse framework for combining
  multiple partitions.
\newblock {\em Journal of Machine Learning Research}, 3:582--617, 2002.

\bibitem{Luxburg2007}
U.~von Luxburg.
\newblock A tutorial on spectral clustering.
\newblock {\em Statistics and Computing}, pages 395--416, 2007.

\bibitem{SNF2014}
B.~Wang, A.~Mezlini, F.~Demir, M.~Fiume, Z.~Tu, M.~Brudno, B.~Haibe-Kains, and
  A.~Goldenberg.
\newblock {Similarity network fusion for aggregating data types on a genomic
  scale}.
\newblock {\em Nature Methods}, 11:333--337, 2014.

\bibitem{WeinbergerSaul2009}
K.~Q. Weinberger and L.~K. Saul.
\newblock Distance metric learning for large margin nearest neighbor
  classification.
\newblock {\em Journal of Machine Learning Research}, 10:207--244, 2009.

\bibitem{XingNgJordanRussell2002}
E.~Xing, A.~Y. Ng, M.~I. Jordan, and S.~Russell.
\newblock Distance metric learning, with application to clustering with
  side-information.
\newblock In {\em Proceedings of Neural Information Processing Systems (NIPS)},
  pages 521--528, 2002.

\bibitem{XuTaoXu2013}
C.~Xu, D.~Tao, and C.~Xu.
\newblock A survey on multi-view learning.
\newblock {\em arXiv:1304.5634}, 2013.

\bibitem{CF}
D.~Yan, A.~Chen, and M.~I. Jordan.
\newblock Cluster {F}orests.
\newblock {\em Computational Statistics and Data Analysis}, 66:178--192, 2013.

\bibitem{YanDavis2018}
D.~Yan and G.~E. Davis.
\newblock {The turtleback diagram for conditional probability}.
\newblock {\em The Open Journal of Statistics}, 8(4):684--705, 2018.

\bibitem{YanHuangJordan2009}
D.~Yan, L.~Huang, and M.~I. Jordan.
\newblock Fast approximate spectral clustering.
\newblock In {\em Proceedings of the 15th ACM SIGKDD}, pages 907--916, 2009.

\bibitem{rsDiagnosis2019}
D.~Yan, C.~Li, N.~Cong, L.~Yu, and P.~Gong.
\newblock A structured approach to the analysis of remote sensing images.
\newblock {\em International Journal of Remote Sensing}, 40(20):7874--7897,
  2019.

\bibitem{deepTacoma2019}
D.~Yan, T.~W. Randolph, J.~Zou, and P.~Gong.
\newblock Incorporating deep features in the analysis of tissue microarray
  images.
\newblock {\em Statistics and Its Interface}, 12(2):283--293, 2019.

\bibitem{rpForests2019}
D.~Yan, Y.~Wang, J.~Wang, H.~Wang, and Z.~Li.
\newblock K-nearest neighbor search by random projection forests.
\newblock {\em IEEE Transactions on Big Data}, PP:1--12, 2019.

\bibitem{YangJin2006}
L.~Yang and R.~Jin.
\newblock Distance metric learning: A comprehensive survey.
\newblock {\em Michigan State Universiy}, 2, 2006.

\end{thebibliography}
%%
%%

\end{document}